\documentclass[10pt,twocolumn,letterpaper]{article}

\usepackage{cvpr}              

\usepackage{graphicx}
\usepackage{amsmath}
\usepackage{amssymb}
\usepackage{booktabs}
\usepackage{amsthm}
\newtheorem{theorem}{Theorem}

\usepackage[utf8]{inputenc} 
\usepackage[T1]{fontenc}    
\usepackage{url}            
\usepackage{amsfonts}       
\usepackage{nicefrac}       
\usepackage{microtype}      

\usepackage{breqn}
\usepackage{float}
\usepackage{caption}
\usepackage{subcaption}
\usepackage[cal=cm]{mathalfa}
\usepackage{comment}
\usepackage{mathtools}
\usepackage{bbm}

\usepackage{multirow}
\usepackage{threeparttable}
\usepackage[table,xcdraw]{xcolor}
\usepackage{pifont}
\usepackage{tabularray}
\usepackage[misc]{ifsym}

\usepackage[pagebackref,breaklinks,colorlinks]{hyperref}

\usepackage[capitalize]{cleveref}
\crefname{section}{Sec.}{Secs.}
\Crefname{section}{Section}{Sections}
\Crefname{table}{Table}{Tables}
\crefname{table}{Tab.}{Tabs.}

\def\cvprPaperID{9321} 
\def\confName{CVPR}
\def\confYear{2023}

\begin{document}

\title{NoisyQuant: Noisy Bias-Enhanced Post-Training Activation Quantization for \\Vision Transformers}

\author{
Yijiang Liu$^{*1}$, Huanrui Yang$^{*2}$, Zhen Dong$^2$, Kurt Keutzer$^2$, Li Du$^1$\textsuperscript{\Letter}, Shanghang Zhang$^3$\textsuperscript{\Letter} \\
$^1${\normalsize Nanjing University}, \
$^2${\normalsize University of California, Berkeley}\\
$^3${\normalsize National Key Laboratory for Multimedia Information Processing, School of Computer Science, Peking University}\\
{\tt\small liuyijiang@smail.nju.edu.cn, \{huanrui, zhendong, keutzer\}@berkeley.edu}\\
{\tt\small ldu@nju.edu.cn, shanghang@pku.edu.cn }
}

\maketitle
\newcommand\blfootnote[1]{%
\begingroup
\renewcommand\thefootnote{}\footnote{#1}%
\addtocounter{footnote}{-1}%
\endgroup
}

\blfootnote{* Equal contribution.}
\blfootnote{\textsuperscript{\Letter} Corresponding Author.}

\begin{abstract}
   The complicated architecture and high training cost of vision transformers urge the exploration of post-training quantization. 
   However, the heavy-tailed distribution of vision transformer activations hinders the effectiveness of previous post-training quantization methods, even with advanced quantizer designs. 
   Instead of tuning the quantizer to better fit the complicated activation distribution, this paper proposes NoisyQuant, a quantizer-agnostic enhancement for the post-training activation quantization performance of vision transformers.
   We make a surprising theoretical discovery that for a given quantizer, adding a fixed Uniform noisy bias to the values being quantized can significantly reduce the quantization error under provable conditions.
   Building on the theoretical insight, NoisyQuant achieves the first success on actively altering the heavy-tailed activation distribution with additive noisy bias to fit a given quantizer. 
   Extensive experiments show NoisyQuant largely improves the post-training quantization performance of vision transformer with minimal computation overhead. For instance, on linear uniform 6-bit activation quantization, NoisyQuant improves SOTA top-1 accuracy on ImageNet by up to 1.7\%, 1.1\% and 0.5\% for ViT, DeiT, and Swin Transformer respectively, achieving on-par or even higher performance than previous nonlinear, mixed-precision quantization. 
\end{abstract}

\section{Introduction}
\label{sec:intro}

\begin{figure*}[t]
    \centering
    \includegraphics[width=0.95\textwidth]{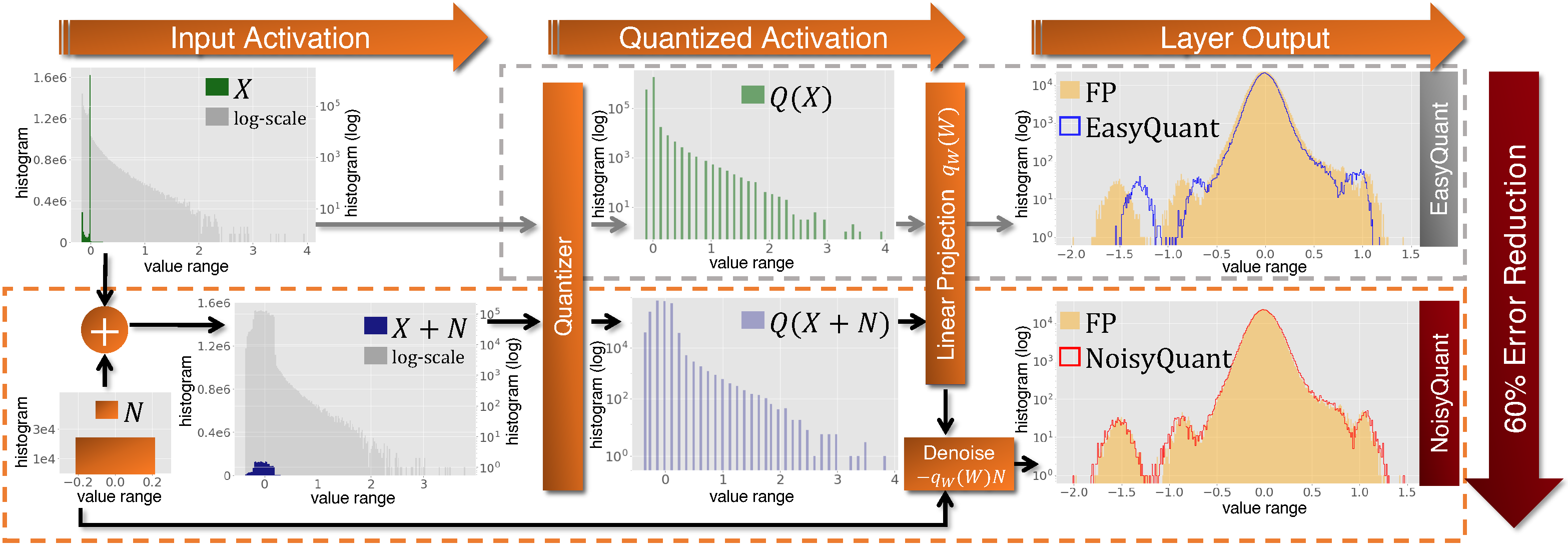}
    \caption{\textbf{Overview of the NoisyQuant pipeline (orange box) with comparison to EasyQuant~\cite{Wu2020EasyQuantPQ} (gray box).} Histograms of input activations, quantized activations, and layer outputs are illustrated in the pipeline. Starting from the input activation $X$ following a GELU function, where the green histogram is plotted in linear-scale and the gray in log-scale, NoisyQuant adds a fixed NoisyBias $N$ sampled from a Uniform distribution onto $X$. Adding the noise before quantization flattens the peaks in the activation distribution to make it more friendly to quantization, as visualized in the linear/log-scaled histogram of $X+N$, and theoretically proved in~\cref{ssec:theo}. The noise is removed from the result of the fixed-point activation-weight multiplication with a denoising bias term, which we formulate in~\cref{ssec:dep}. The rightmost sub-figures illustrate the layer output distribution of EasyQuant and NoisyQuant. Compared to the output of the floating-point model (yellow), NoisyQuant output follows the distribution closely and achieves up to 60\% output error reduction on some transformer layers, as shown in~\cref{sec:ablation}. The resulting PTQ performance improvement of the entire model is provided in~\cref{sec:exp}.}
    \label{fig:NoisyDistribution}
    \vspace{-10pt}
\end{figure*}

Inspired by the success of Self Attention (SA)-based transformer models in Natural Language Processing (NLP) tasks~\cite{vaswani2017attention}, recent researches make significant progress in applying transformer models to the field of computer vision~\cite{dosovitskiy2020image,Touvron2021TrainingDI,Liu2021SwinTH,carion2020end}. In the meantime, the typical design of transformer models induces large model sizes, high computational consumption, and long training time. For instance, the widely used DeiT-Base model~\cite{Touvron2021TrainingDI} contains 86M parameters, logs 18G floating-point operations for a single input, and requires 300 epochs of training on the ImageNet dataset. This leads to significant difficulties in hardware deployment. In facing such difficulty, a number of compression and acceleration methods are applied to vision transformer models, including pruning~\cite{chen2021chasing,yang2021nvit}, quantization~\cite{liu2021post,yuan2021ptq4vit}, and neural architecture search~\cite{chen2021autoformer},~\etc.

Among these methods, quantization appears as one of the most effective and widely applied ways~\cite{gholami2021survey}.
Quantization process uses a predefined ``quantizer'' function to convert the continuous representation of weights and activations into a small number of discrete symbols, therefore enabling low-precision representations for straightforward memory savings. For DNN models, the approximation error made by the quantizer inevitably leads to performance drop. 
A series of work focuses on Quantization-Aware Training (QAT) that finetunes the quantized model at low precision~\cite{zhou2016dorefa,polino2018model,dong2019hawq,dong2020hawq,xiao2022csq}. However, given the high training cost and the complicated computation graph of vision transformer models, retraining the model at low precision could be costly and unstable~\cite{yuan2021ptq4vit}.
Alternatively, Post-Training Quantization (PTQ) is preferable for vision transformers as it eliminates the need for re-training or finetuning the quantized model, instead only adjusts the design of the quantizer based on the full-precision pretrained model and a small set of sampled calibration data~\cite{Wu2020EasyQuantPQ,Cai2020ZeroQAN,Liu2021PostTrainingQF,Wang2020TowardsAP,Bhalgat2020LSQIL}. 

For example, linear quantizers~\cite{choi2018pact,shen2020q,yang2021bsq,Wu2020EasyQuantPQ} reduce quantization error by shifting, clipping, and scaling the values to be quantized. Nonlinear quantizers~\cite{zhang2018lq,han2015deep,yuan2021ptq4vit} further adjust the width and location of each quantization bin to better fit the distribution. 

Unfortunately, though progresses are made in designing better PTQ quantizers, it still appears to be significantly challenging to quantize vision transformer models, especially for activation quantization. 
Transformer layers produce up to millions of activation values with sophisticated distributions.
For instance, outputs of the GELU function~\cite{hendrycks2016gaussian} are asymmetrically distributed, with spikes in the histogram at some values and a long tail in a large range (see top-left of~\cref{fig:NoisyDistribution}).
Some linear projection layers also lead to significantly large activation values that are sparsely distributed in a very long tail~\cite{Liu2021PostTrainingQF}.
Consequently, low-precision PTQ on vision transformer suffers from performance degradation, even utilizing non-linear or mixed-precision quantizers at the cost of additional data communication and computation overheads~\cite{Liu2021PostTrainingQF,Yuan2022PTQ4ViTPQ}. \textit{No} linear uniform PTQ method achieves good performance on vision transformer models.

This paper provides a brand new perspective in dealing with the complicated vision transformer activation distribution. Instead of adding more tricks in the quantizer design to fit the activation distribution, this work explores the potential of actively and cost-effectively altering the distribution being quantized, making it more friendly to a given quantizer.
The pipeline of our proposed method is illustrated in~\cref{fig:NoisyDistribution}.
Specifically, we make a surprising discovery that for any quantizer, the quantization error can be significantly reduced by adding a fixed noisy bias sampled from a Uniform distribution to the activation before quantization. We theoretically prove the condition when the quantization error reduction can be achieved. 
On this basis, we propose \textbf{NoisyQuant}, a \textit{plug-and-play}, quantizer-agnostic enhancement on the post-training activation quantization of vision transformer models. For each layer, we sample a Noisy Bias based on the input activation distribution following our theoretical analysis, and compute the corresponding denoising bias to retain the correct output of the layer. At inference time, the Noisy Bias is added to the input activation before quantization, and the denoising bias is removed from the layer output. This process significantly reduces the quantization error with minimal computation overhead.
NoisyQuant leads to significant improvement in the PTQ performance of state-of-the-art vision transformers. Applying NoisyQuant on top of a uniform linear quantization achieves on-par performance to SOTA mixed-precision, nonlinear PTQ methods~\cite{yuan2021ptq4vit,Liu2021PostTrainingQF}. Adding NoisyQuant on top of these nonlinear quantizers achieves further performance gain.

To the best of our knowledge, this paper makes the following novel contributions:
\begin{itemize}
    \item Theoretically shows the possibility and proves feasible conditions for reducing the quantization error of heavy-tailed distributions with a fixed additive noisy bias;
    \item Proposes NoisyQuant, a quantizer-agnostic enhancement for PTQ performance on activation quantization. NoisyQuant achieves the first success in actively refining the distribution being quantized to reduce quantization error following the theoretical results on additive noisy bias, with minimal computation overhead; 
    \item Demonstrates consistent performance improvement by applying NoisyQuant on top of existing PTQ quantizers. For 6-bit PTQ, NoisyQuant improves the ImageNet top-1 accuracy of uniform linear quantized vision transformer models by up to 1.7\%, and improves SOTA nonlinear quantizer PTQ4ViT~\cite{Yuan2022PTQ4ViTPQ} by up to 0.7\%.
\end{itemize}

\section{Related work}
\label{sec:background}

\subsection{Transformer models in computer vision}

Since Self-Attention (SA)~\cite{vaswani2017attention}-based transformer models have shown impressive performance in Natural Language Processing tasks, researchers transfer the idea into computer vision. Pioneering work on Vision Transformer~\cite{dosovitskiy2020image} for the first time builds a totally SA-based model in the field of computer vision and shows its effectiveness. Since then, a variety of transformer-based vision models have emerged with boosted performance, claiming state-of-the-art in multiple computer vision tasks. Notable advances include additional data augmentation and distillation~\cite{Touvron2021TrainingDI}, incorporating multi-stage architecture~\cite{Liu2021SwinTH,wang2021pyramid}, and explore novel attention design~\cite{Liu2021SwinTH,yuan2021tokens,hatamizadeh2022global}. 
Transformer models also advance to tasks beyond classification, such as object detection~\cite{carion2020end}, semantic segmentation~\cite{SETR}, and image generation~\cite{lee2021vitgan}. Furthermore, the development of vision transformer brings unification of model architecture and representation across image and text, enabling stronger vision-language connection via image-text-pair contrastive learning~\cite{Radford2021LearningTV}.

The rapid development of vision transformer models motivates the exploration of quantization, as quantization brings straightforward memory and computation savings agnostic to model architectures.
Compared to CNN models, the use of advanced activation functions like GELU~\cite{hendrycks2016gaussian} and more complicated computation graphs like the attention mechanism makes vision transformers more sensitive to quantization~\cite{Yuan2022PTQ4ViTPQ}. This work aims to resolve the difficulty of post-training vision transformer activation quantization caused by the heavy-tailed activation distribution, which we believe will facilitate the deployment of state-of-the-art transformer models into real-world computer vision applications.

\subsection{Post-training quantization methods}

Post-training quantization (PTQ) is a preferable quantization method in scenarios with limited training data or limited computational resources for the expensive quantization-aware fine-tuning process. Previous methods have dedicatedly studied PTQ on CNNs. For instance, EasyQuant~\cite{Wu2020EasyQuantPQ} presents a quick searching mechanism to determine the proper clipping range for quantization. ZeroQ~\cite{Cai2020ZeroQAN} proposes a distillation mechanism to generate proxy input images, which can leverage inherent statistics of batch normalization layers to conduct PTQ. SQuant~\cite{guo2022squant} calibrates the model for lower quantization error on the basis of the sensitivity obtained by the Hessian spectrum. On vision transformers, Liu~\etal~\cite{Liu2021PostTrainingQF} presents a linear PTQ method that uses Pearson correlation coefficients and ranking loss to determine scaling factors. PTQ4ViT~\cite{Yuan2022PTQ4ViTPQ} introduces a nonlinear Twin Uniform Quantization based on vision transformers' activation distributions, which set different scaling factors for 1) positive and negative activations of GeLU, and 2) small and large values of Softmax, which can reduce the activation quantization error to some extent with the cost of additional computation overhead.
\cite{wei2022outlier,li2022repq} reparameterize the LN layer to suppress outliers by scaling down activation values.

Unlike previous PTQ methods that analyze the activation distribution and fit the quantizer accordingly, our method takes a novel perspective of actively modifying the distribution being quantized with a sampled noisy bias. We show our proposed method can bring quantization error reduction and performance enhancement for all the PTQ quantizers for vision transformer activation quantization. 
\section{Method}
\label{sec:method}

In this section, we provide a preliminary introduction to the notation and typical formulation of DNN quantization in~\cref{ssec:form}, theoretically analyze the potential of reducing quantization error with pre-sampled additive noisy bias in~\cref{ssec:theo}, and formulate NoisyQuant, a noisy bias-enhanced post-training activation quantization method with reduced quantization error in~\cref{ssec:dep}.

\subsection{Preliminary}
\label{ssec:form}

To start with, we recap the post-training quantization process on DNN models. Given the dominance of fully-connected (FC) linear projection layers in transformer models, here we focus our discussion on the FC layer.

The computation performed in a FC layer with weight $W\in \mathbb{R}^{k\times m}$, input activation $X\in \mathbb{R}^{m\times n}$, and bias $B\in \mathbb{R}^{k\times 1}$ can be formulated as 
\begin{equation}
\label{equ:linear}
    f(X) = WX + B.
\end{equation}
The main computation burden in~\cref{equ:linear} lies in the matrix multiplication of $WX$, which requires a total of $k \times m \times n$ multiply-accumulate computations (MACs). Post-training quantization aims to reduce the computation cost of matrix multiplication by converting $X$ and $W$ to fixed-point quantized values, therefore replacing floating-point MACs with much cheaper fixed-point operations~\cite{horowitz20141,yao2021hawq}.
The quantized FC layer can therefore be represented as
\begin{equation}
\label{equ:quantlinear}
    f_q(X) = q_W(W) q_A(X) + B,
\end{equation}
where $q_W(\cdot)$ and $q_A(\cdot)$ denotes the quantizer function of weight and activation respectively. 
Previous research has observed that the heave-tailed activation distribution of $X$ in transformer is causing significant quantization error between $X$ and $q_A(X)$ at low precision, leading to significant performance degradation.
Previous PTQ methods modify the design of quantizer $q_A(\cdot)$ to mitigate the  quantization error. In this paper, we propose an alternative approach, where we modify the distribution of activation $X$ with a pre-sampled noisy bias before quantization, and prove it to be a plug-and-play enhancement on any quantizer design.

\subsection{Theoretical analysis on quantization error}
\label{ssec:theo}

As previous research mainly blames the heavy-tailed distribution of transformer activation as the main cause for large quantization errors~\cite{yuan2021ptq4vit,Liu2021PostTrainingQF}, in this section we analyze how the quantization error change as we actively alter the activation distribution being quantized. A straightforward way to alter the input activation distribution is to add a fixed ``Noisy Bias'' sampled from a Uniform random distribution. 
Denoting the input activation as $X$ and the Noisy Bias as $N$, where $N$ and $X$ have the same dimensions, here we formulate the quantization error difference between $X$ and $X+N$ for a quantizer $Q$ in~\cref{equ:errordiff}
\begin{equation}
\begin{split}
\label{equ:errordiff}
    &D(X,N) = QE(X+N)-QE(X) =\\
    &||(Q(X+N)-X-N)||_2^2 - ||(Q(X)-X)||_2^2.
\end{split}
\end{equation}
\cref{the:QE} states the condition where $D(x,N) \leq 0$ holds for each histogram snapshot $x$ of activation $X$.

\begin{theorem}
\label{the:QE}
Consider a given quantizer $Q$ with quantization bin width $2b$. For each histogram snapshot of $X$ where all elements $X_i$ have the same distance $x$ away from the center of the quantization bin, and for a Noisy Bias $N$ sampled from $N \sim \mathcal{U}(-n,n)$ where $x \leq n \leq 2b-x$, we have 
\begin{equation}
\label{equ:condition}
    D(x,N) \leq 0 \ \ \text{\textit{iff}} \ \ 0 \leq x \leq n \left(1 - \sqrt{\frac{n}{3b}}\right).
\end{equation}
\end{theorem}

\begin{proof}
Without loss of generality, we consider the quantization bin of quantizer $Q$ which $x$ falls into to be centered at the 0 point, therefore $0\leq x \leq b$. 

In this way, $Q(X_i)=b$ for all elements in the histogram snapshot, leading to a quantization error
\begin{equation}
\label{equ:QE1}
    QE(X) = (b-x)^2.
\end{equation}

Consider adding $N \sim \mathcal{U}(-n,n)$. $X+N$ will follow $\mathcal{U}(x-n,x+n)$.
In the case of $x \leq n \leq 2b-x$, $Q(X_i+N_i) = b$ if $N_i\in [-x,n]$, and $Q(X_i+N_i) = -b$ if $N_i\in [-n,-x)$ for all elements in $X+N$. Though we only sample a single instance of $N$ from the Uniform distribution, given the large number of elements in the histogram snapshot, the empirical quantization error $QE(X+N)$ can be estimated with an expectation over all $N \sim \mathcal{U}(-n,n)$ following the Weak Law of Large Numbers, as
\begin{equation}
\begin{split}
\label{equ:QE2}
    &\mathbf{E}_N [QE(X+N)] = \frac{1}{2n} \int_{x-n}^{x+n} QE(x+z) \,dz \\
    &= \frac{1}{2n} \left[ \int_{x-n}^{0} (z+b)^2 \,dz + \int_{0}^{x+n} (z-b)^2 \,dz \right] \\
    &= x^2 - \frac{b}{n}x^2 +\frac{n}{3}-nb+b^2.
\end{split}
\end{equation}

Combining \cref{equ:QE1} and \cref{equ:QE2}, we have
\begin{equation}
\begin{split}
\label{equ:error}
    &D(x,N) = \mathbf{E}_{N}\left[ QE(X+N)\right] -QE(X) \\
    & = -\frac{b}{n}x^2 + 2bx +\frac{n^2}{3} -nb.
\end{split}
\end{equation}
We verify this derivation empirically in~\cref{sec:ablation}.

It can be observed that given $b$ and $n$, $D(x,N)$ is a quadratic function with respect to the activation value $x$. The inequality $D(x,N) \leq 0, 0\leq x \leq n$ can be easily solved with respect to $x$ as
\begin{equation}
    0 \leq x \leq n \left(1 - \sqrt{\frac{n}{3b}}\right),
\end{equation}
which is always feasible given both $b$ and $n$ are positive.
\end{proof}

\cref{the:QE} indicates that adding Noisy Bias can always reduce the quantization error of elements close to the center of the quantization bin, which is the source of large quantization error in a heavy-tailed distribution.

In practice, to choose a proper $n$ for a layer in a pretrained model, we can first acquire the empirical distribution of activation $X$ by passing a small amount of calibration data through the pretrained model. With the distribution of $X$, we can estimate the expected quantization error reduction of the layer numerically as a function of $n$
\begin{equation}
\label{equ:obj}
    \mathcal{L}(n) = \sum_{x\in X} [D(x,N)].
\end{equation}
Note that $x$ here denotes the distance between each activation value and the corresponding quantization bin center.
Though we cannot find a closed-form solution for the minima of $\mathcal{L}(n)$, we can easily find a good-enough choice with a linear search over a small range of $n$.

\subsection{NoisyQuant formulation and pipeline}
\label{ssec:dep}

Given our theoretical analysis on reducing quantization error with Noisy Bias, we propose to apply the Noisy Bias as an enhancement to the PTQ methods of vision transformer models, which we name as ``\textbf{NoisyQuant}''. The pipeline of NoisyQuant is visualized in~\cref{fig:NoisyDistribution}. 
Specifically, before performing PTQ and deploying the model, we sample a single Noisy Bias $N\in \mathbb{R}^{m\times 1}$ for each layer from a Uniform distribution, and fix it throughout the inference. The range of $N$ is determined by the calibration objective defined in~\cref{equ:obj}. 
At inference time, we add the Noisy Bias $N$ to the input activation $X$ under the broadcasting rule before going through the activation quantizer, therefore actively altering the distribution being quantized to reduce the input quantization error. 
After the activation-weight multiplication in the linear layer, we remove the impact of $N$ by adjusting the bias term in the linear layer with a denoising bias computed from $N$, therefore retrieving the correct output. Formally, NoisyQuant converts the quantized FC computation defined in~\cref{equ:quantlinear} into 
\begin{equation}
\label{equ:noisyquant}
    f_{Nq}(X) = q_W(W) q_A(X+N) + \left(B-q_W(W) N\right).
\end{equation}

Since $N$ is sampled before the deployment of the model and fixed it throughout the inference, the denoising output bias $B' = B-q_W(W) N$ only needs to be computed once before the deployment, without any computation overhead in the inference process. 
The only computation overhead brought by NoisyQuant at inference time is the on-the-fly summation of $X+N$, which is negligible comparing to the cost of the matrix multiplications. 
Both $B'$ and $N$ shall be stored with higher precision. In practice, we store both variables as INT16 to enable integer-only inference. 

Comparing to the output of the floating-point FC layer defined in~\cref{equ:linear}, the output quantization error induced by the activation quantization defined in~\cref{equ:quantlinear} is
\begin{equation}
\label{equ:quanterror}
    QE_O(X) = ||f_q(X) - f(X)||_2^2 = ||W[Q(X)-X]||_2^2,
\end{equation}
where $Q(\cdot)$ is short of $q_A(\cdot)$, and we omit the weight quantization $q_W(\cdot)$ as we are focusing on the activation quantization. Similarly, the output quantization error after the application of NoisyQuant can be computed as 
\begin{equation}
\begin{split}
\label{equ:quanterror_noise}
    QE_O'(X) &= ||f_{Nq}(X) - f(X)||_2^2 \\
    &= ||W[Q(X+N)-X-N]||_2^2.
\end{split}
\end{equation}

As we prove in~\cref{the:QE} that the Noisy Bias enables $X+N$ to have a lower quantization error than $X$, we can achieve lower output quantization error in~\cref{equ:quanterror_noise} than in~\cref{equ:quanterror} when NoisyQuant is applied, therefore improving the PTQ performance of the model. 

\section{Theoretical insight verification}
\label{sec:ablation}

In this section, we provide empirical verification of our theoretical analysis on reducing the activation quantization error with NoisyQuant. We numerically verify~\cref{the:QE} with simulated data in~\cref{ssec:verify}, and demonstrate the reduction in both input and output quantization errors under the true inference-time activation distribution in~\cref{ssec:layererror}.

\subsection{Empirical verification of~\cref{the:QE}}
\label{ssec:verify}

Here we follow the settings in~\cref{the:QE} to empirically verify its theoretical derivation. Specifically, we set the quantization bin range $b=1$, and explore how the quantization error difference induced by the Noisy Bias change with different choice of activation value $x$ and noisy bias range $n$. For all empirical results we experiment with 10 instances of independently sampled Noisy Bias $N$, and report the mean and standard derivation of $D(X,N)$ defined in~\cref{equ:errordiff} across the 10 instances. We consider input activation $X$ to be a tensor with 20 dimensions. Given the tens to hundreds of thousands of activation values in each transformer layer, it is likely to see more than 20 activation elements taking the same value. As we base our theoretical derivation on the Weak Law of Large Numbers, having more elements taking the same value will lead to less variance than the simulation results provided in this section.

\begin{figure}[tb]
    \centering
    \includegraphics[width=0.9\linewidth]{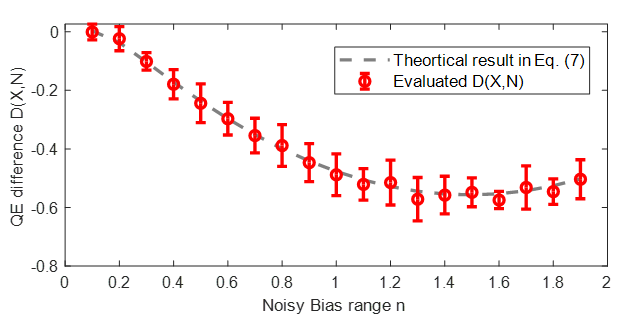}
    \caption{Verifying~\cref{equ:error} with respect to Noisy Bias range $n$. Circle indicates the mean and error bar the std for evaluated results.}
    \label{fig:verify_N}
    \vspace{-10pt}
\end{figure}

For the first experiment, we fix all elements in $X$ to take value $x=0.1$, and alter $n$ in the range of $[0.1,1.9]$. We compare the empirical $D(X,N)$ and the theoretical result derived in~\cref{equ:error} in~\cref{fig:verify_N}. The evaluated results closely follow the theoretical line. Even with only 20 activation elements the standard deviation across independently sampled Noisy Bias is much smaller than the quantization error benefit brought by adding the Noisy Bias.

\begin{figure}[tb]
    \centering
    \includegraphics[width=0.9\linewidth]{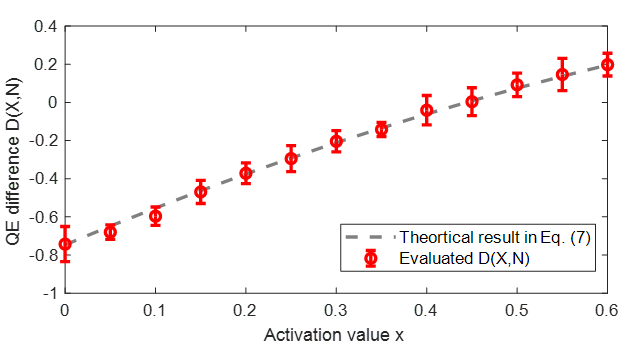}
    \caption{Verifying~\cref{equ:error} with respect to activation value $x$. Circle indicates the mean and error bar the std for evaluated results.}
    \label{fig:verify_X}
    \vspace{-10pt}
\end{figure}

Similarly, we fix $n=1.4$ and alter activation value $x$ in the range of $[0,0.6]$. The results are shown in~\cref{fig:verify_X}, where we see the same trend of close mean and small variance as in the previous experiment. Quantization error can be reduced by Noisy Bias when $x<0.44$, which follows the numerical computation result of~\cref{equ:condition}. Both results show the correctness and stability of our derivation of~\cref{equ:error} and the proof of~\cref{the:QE}.

\subsection{Quantization errors in transformer layers}
\label{ssec:layererror}

\begin{table}
\centering
\caption{\textbf{Averaged quantization error on different types of layers.} The calculation is performed with 5,120 images randomly selected from the ImageNet validation set on the Swin-T model. Layer names are defined as same as in Timm~\cite{rw2019timm}. We compute the input quantization error difference $D$ and compare the output errors between NoisyQuant and EasyQuant.}
\label{qe_drop}
\resizebox{0.85\linewidth}{!}{
\begin{tblr}{
  cells = {c},
  vline{2-3} = {-}{},
  hline{1,3,7} = {-}{},
}
\textbf{Layer} & \textbf{\cref{equ:errordiff}} & \textbf{\cref{equ:quanterror}} & \textbf{\cref{equ:quanterror_noise}} & \textbf{Output} \\
\textbf{Name}  & $D(X)$             & $QE_O(X)$                 & $QE_O'(X)$                       &  \textbf{QE drop}                \\
qkv            & -0.029             & 0.107                   & 0.096                          & 10\%             \\
proj           & -0.106             & 0.093                   & 0.081                          & 13\%             \\
fc1            & -0.527             & 0.126                   & 0.123                          & 2\%              \\
fc2            & -5.130             & 1.049                   & 0.852                          & 19\%             
\end{tblr}
}
\end{table}

\begin{figure}[tb]
    \centering
    \includegraphics[width=0.8\linewidth]{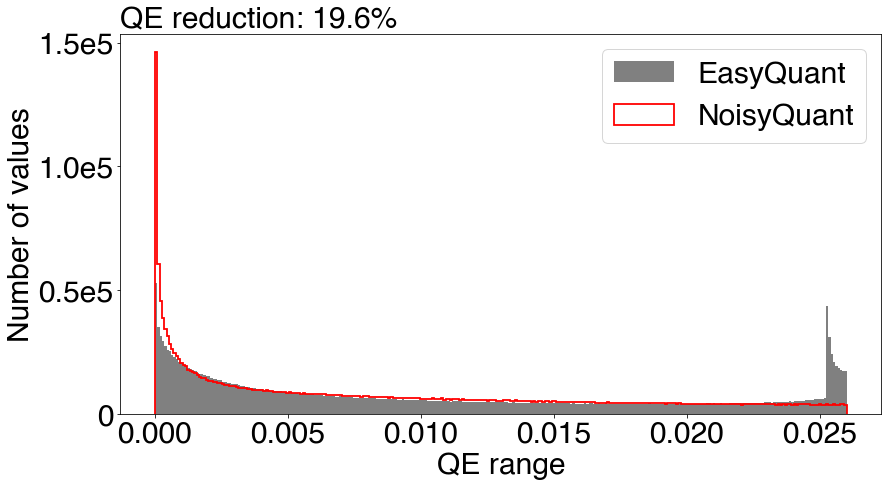}
    \vspace{-8pt}
    \caption{Quantization error histogram of fc2 input activation. We apply NoisyQuant on top of the EasyQuant quantizer, and compare the quantization error distribution (red) with EasyQuant (grey). }
    \label{fig:gelu_noise}
    \vspace{-10pt}
\end{figure}

\begin{figure}[htb]
    \centering
    \begin{minipage}[t]{\linewidth}
		\centering
    \includegraphics[width=\linewidth]{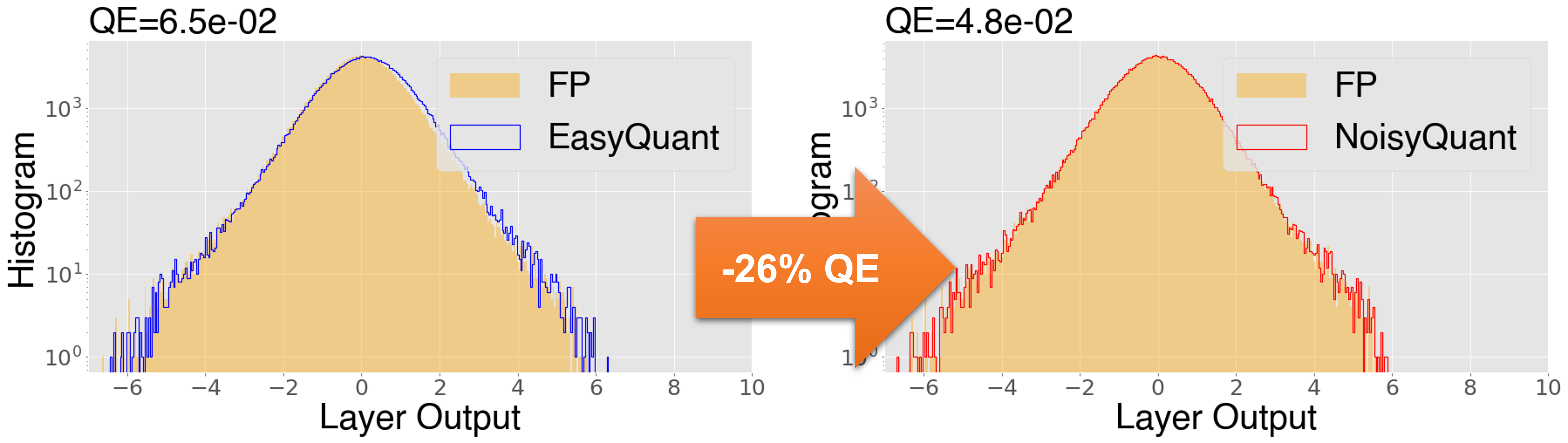}
		\subcaption{26\% $QE_O$ reduction on qkv layer} 
	\end{minipage}
    \begin{minipage}[t]{\linewidth}
		\centering
    \includegraphics[width=\linewidth]{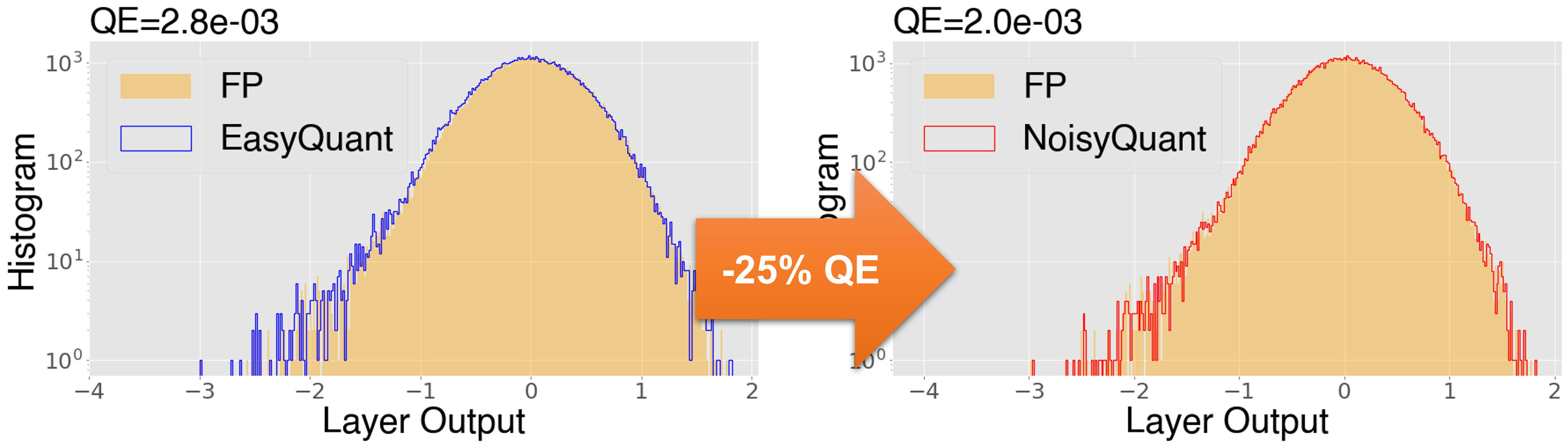}
		\subcaption{25\% $QE_O$ reduction on proj layer} 
	\end{minipage}
    \begin{minipage}[t]{\linewidth}
		\centering
    \includegraphics[width=\linewidth]{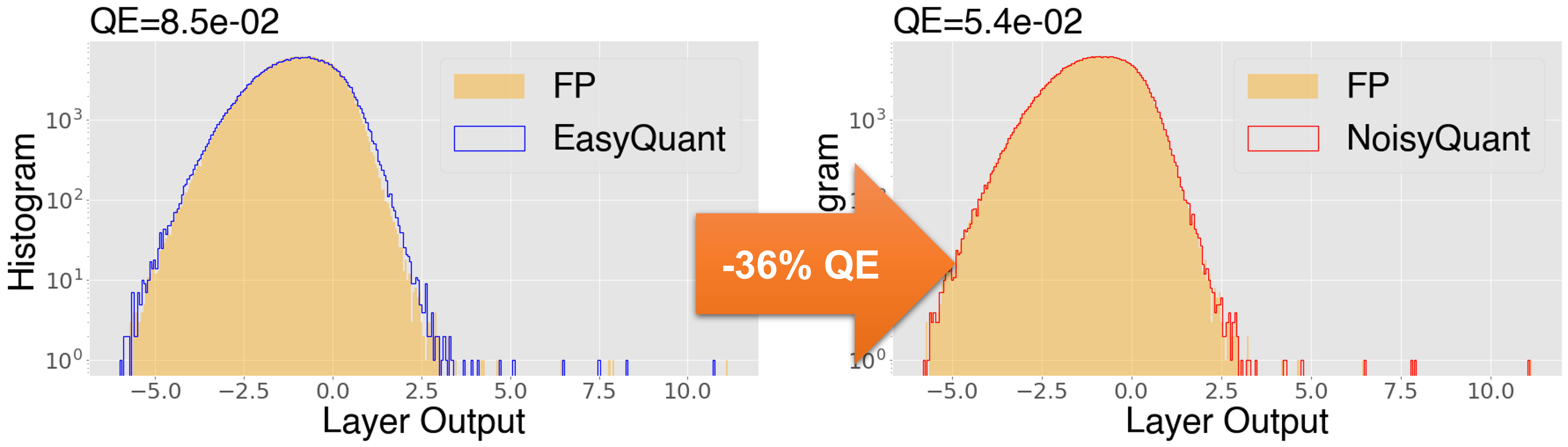}
		\subcaption{36\% $QE_O$ reduction on fc1 layer} 
	\end{minipage}
    \begin{minipage}[t]{\linewidth}
		\centering
    \includegraphics[width=\linewidth]{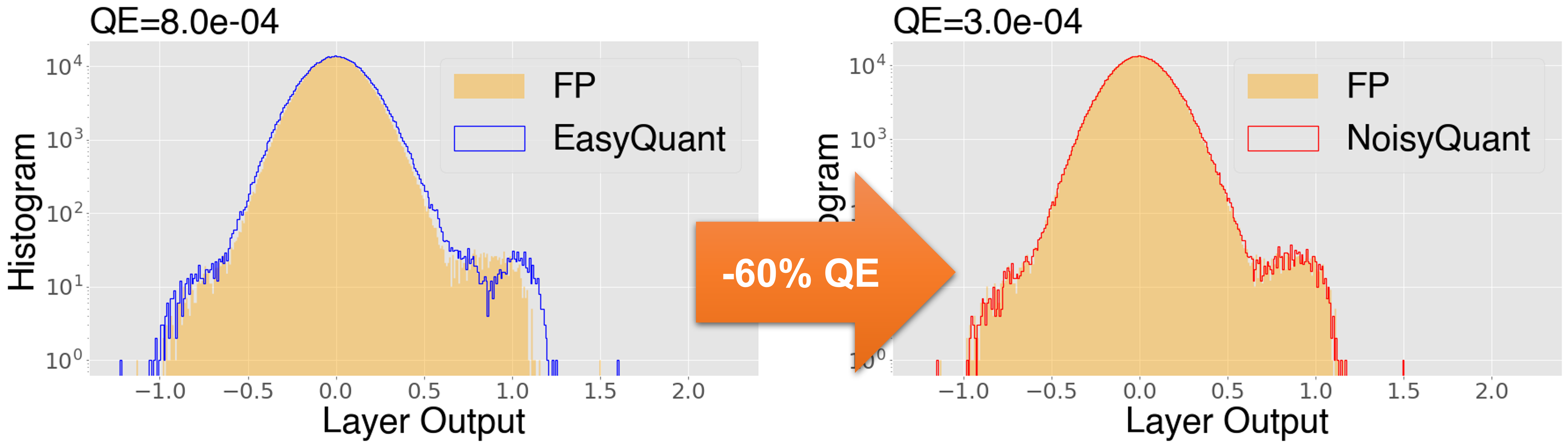}
		\subcaption{60\% $QE_O$ reduction on fc2 layer} 
	\end{minipage}
    \vspace{-8pt}

    \caption{Output histogram in selected transformer layers before and after activation quantization. NoisyQuant significantly reduces the output quantization error over EasyQuant baselines. }
    \label{fig:noisy_dist}
    \vspace{-10pt}
\end{figure}

As mentioned in~\cref{ssec:theo}, NoisyQuant enables the reduction of both input activation quantization error ($QE$) and output error ($QE_O$) of a linear layer given a quantizer. In this section, we verify the error reduction on the real activation distribution of vision transformer model. 
Practically, we run evaluation on randomly selected 5120 images of validation set with linear PTQ quantizer EasyQuant~\cite{Wu2020EasyQuantPQ} as baseline, and calculate $QE$ and $QE_O$ with or without adding Noisy Bias. Experiments are ran on different types of layers, namely qkv, proj, fc1, and fc2, respectively. The results are averaged across all layers of each type. As shown in~\cref{qe_drop}, input quantization error difference $D(X)$ defined in~\cref{equ:errordiff} is consistently lower than 0, which meets our expectation that Noisy Bias reduces input quantization error. This further leads to significant drop in output error between the quantized and floating-point layers for all layer types. Up to 19\% averaged output error drop can be achieved in fc2 layers. 

To further understand the impact NoisyQuant makes, we visualize both the input and output distribution of each layer. We start with highlighting fc2 as it achieves the greatest QE reduction. The fc2 layer takes in the output of GELU activations, which has significantly complicated distribution as introduced in~\cref{sec:intro}. 
\cref{fig:gelu_noise} visualizes the distribution of input quantization error with or without NoisyQuant. With EasyQuant only, significant amount of activation elements are showing large quantization error, which is most probably owing to bad quantization on densely distributed areas. Instead, NoisyQuant generally makes QE distributed plainly and for most elements near zero, therefore leading to significant QE reduction overall. 
Subsequently, we plot the histogram of selected layers' outputs in~\cref{fig:noisy_dist}. NoisyQuant output consistently follows the floating-point layer output distribution closer than that of EasyQuant, especially at values with large magnitudes. This leads to up to 60\% error reduction in the output of some vision transformer layers.

\section{Experiments}
\label{sec:exp}

\begin{table}
\centering
\caption{\textbf{PTQ accuracy of ViT models.} Top-1 accuracy on ImageNet is reported in the table. W/A represents weight and activation bit-width, respectively. The default patch size is 16$\times$16 and the default input resolution is 224$\times$224. ``*'' after model name indicates 384$\times$384 input. ``MP'' represents mixed-precision. For our methods, ``NoisyQuant-Linear'' is implemented on top of the EasyQuant quantizer, whereas ``NoisyQuant-PTQ4ViT'' is implemented on the nonlinear PTQ4ViT quantizer. All NoisyQuant results are shown with mean$\pm$STD of 20 repeated experiments.}
\label{comp_vit}
\resizebox{0.95\linewidth}{!}{
\begingroup
    \setlength{\tabcolsep}{3pt}
\begin{tabular}{ccccc} 
\toprule
\textbf{Method}                                          & \textbf{W/A} & \textbf{ViT-S} & \textbf{ViT-B} & \textbf{ViT-B*}  \\ 
\hline
Pretrained                                                       & 32/32        & 81.39          & 84.54          & 86.00            \\ 
\hline
Percentile~\cite{Li2019FullyQN}                                              & 6/6          & 67.74          & 77.63          & 77.60            \\
Liu~\etal~\cite{Liu2021PostTrainingQF}                                                    & 6 MP          & -              & 75.26              & -                \\
PTQ4ViT-Linear~\cite{Yuan2022PTQ4ViTPQ}                                                 & 6/6          & 70.24          & 75.66          & 46.88   \\
EasyQuant~\cite{Wu2020EasyQuantPQ}                                               & 6/6          & 75.13          & 81.42          & 82.02            \\
\rowcolor[rgb]{1,0.808,0.576} NoisyQuant-Linear                 & 6/6          & \textbf{76.86}\scriptsize{$\pm$0.06}          & \textbf{81.90}\scriptsize{$\pm$0.11}          & \textbf{83.00}\scriptsize{$\pm$0.09}            \\
\hline
PTQ4ViT~\cite{Yuan2022PTQ4ViTPQ}                                                 & 6/6          & 78.63          & 81.65          & \textbf{83.34}   \\
\rowcolor[rgb]{1,0.808,0.576} NoisyQuant-PTQ4ViT & 6/6          & \textbf{78.65}\scriptsize{$\pm$0.07} & \textbf{82.32}\scriptsize{$\pm$0.09} & 83.22\scriptsize{$\pm$0.10}            \\ 
\hline
Percentile~\cite{Li2019FullyQN}                                              & 8/8          & 78.77          & 80.12          & 82.53            \\
Liu~\etal~\cite{Liu2021PostTrainingQF}                                                    & 8 MP          & -              & 76.98              & -                \\
FQ-ViT~\cite{Lin2022FQViTPQ}                                                  & 8/8          & -              & 83.31          & -                \\
PTQ4ViT-Linear~\cite{Yuan2022PTQ4ViTPQ}                                                 & 8/8          & 80.46          & 83.89          & 85.35            \\
EasyQuant~\cite{Wu2020EasyQuantPQ}                                               & 8/8          & 80.75          & 83.89          & 85.53            \\
\rowcolor[rgb]{1,0.808,0.576} NoisyQuant-Linear                 & 8/8          & \textbf{80.81}\scriptsize{$\pm$0.01}          & \textbf{84.10}\scriptsize{$\pm$0.03}          & \textbf{85.56}\scriptsize{$\pm$0.01}            \\
\hline
PTQ4ViT~\cite{Yuan2022PTQ4ViTPQ}                                                 & 8/8          & 81.00          & \textbf{84.25}          & 85.82            \\
\rowcolor[rgb]{1,0.808,0.576} NoisyQuant-PTQ4ViT          & 8/8          & \textbf{81.15}\scriptsize{$\pm$0.02} & 84.22\scriptsize{$\pm$0.01} & \textbf{85.86}\scriptsize{$\pm$0.01}   \\
\bottomrule
\end{tabular}
\endgroup
}
\end{table}

\begin{table}
\centering
\caption{\textbf{PTQ accuracy of DeiT models.} Abbreviations are the same as \cref{comp_vit}.}
\label{comp_deit}
\resizebox{0.95\linewidth}{!}{
\begingroup
    \setlength{\tabcolsep}{3pt}
\begin{tabular}{ccccc} 
\toprule
\textbf{Pretrained}                                          & \textbf{W/A} & \textbf{DeiT-S} & \textbf{DeiT-B} & \textbf{DeiT-B*}  \\ 
\hline
Pretrained                                                       & 32/32        & 79.85           & 81.80           & 83.11             \\ 
\hline
Percentile~\cite{Li2019FullyQN}                                              & 6/6          & 70.49           & 73.99           & 78.24             \\
Bit-Split~\cite{Wang2020TowardsAP}                                              & 6/6          & 74.58           & 76.39           & -                 \\
Liu~\etal~\cite{Liu2021PostTrainingQF}                                                    & 6/6          & 74.58           & 77.02           & -                 \\
Liu~\etal~\cite{Liu2021PostTrainingQF}                                                    & 6 MP          & 75.10           & 77.47           & -                 \\
PTQ4VIT-Linear~\cite{Yuan2022PTQ4ViTPQ}                                                 & 6/6          & 72.26           & 78.78           & 68.44             \\
EasyQuant~\cite{Wu2020EasyQuantPQ}                                               & 6/6          & 75.27           & 79.47           & 81.26             \\
\rowcolor[rgb]{1,0.808,0.576} NoisyQuant-Linear                 & 6/6          & \textbf{76.37}\scriptsize{$\pm$0.13}           & \textbf{79.77}\scriptsize{$\pm$0.08}           & \textbf{81.40}\scriptsize{$\pm$0.06}             \\
\hline
PTQ4VIT~\cite{Yuan2022PTQ4ViTPQ}                                                 & 6/6          & 76.28           & 80.25           & 81.55             \\
\rowcolor[rgb]{1,0.808,0.576} NoisyQuant-PTQ4ViT & 6/6          & \textbf{77.43}\scriptsize{$\pm$0.08}  & \textbf{80.70}\scriptsize{$\pm$0.11}  & \textbf{81.65}\scriptsize{$\pm$0.04}    \\ 
\hline
Percentile~\cite{Li2019FullyQN}                                              & 8/8          & 73.98           & 75.21           & 80.02             \\
Bit-Split~\cite{Wang2020TowardsAP}                                              & 8/8          & 76.39           & 79.42           & -                 \\
Liu~\etal~\cite{Liu2021PostTrainingQF}                                                    & 8/8          & 77.47           & 80.48           & -                 \\
Liu~\etal~\cite{Liu2021PostTrainingQF}                                                    & 8 MP          & 78.09           & 81.29           & -                 \\
FQ-ViT~\cite{Lin2022FQViTPQ}                                                  & 8/8          & 79.17           & 81.20           & -                 \\
PTQ4VIT-Linear~\cite{Yuan2022PTQ4ViTPQ}                                                 & 8/8          & 77.65           & 80.94  & 82.33    \\
EasyQuant~\cite{Wu2020EasyQuantPQ}                                               & 8/8          & 78.98           & 81.19           & 82.10             \\
\rowcolor[rgb]{1,0.808,0.576} NoisyQuant-Linear                 & 8/8          & \textbf{79.11}\scriptsize{$\pm$0.02}           & \textbf{81.30}\scriptsize{$\pm$0.03}           & \textbf{82.23}\scriptsize{$\pm$0.02}             \\
\hline
PTQ4VIT~\cite{Yuan2022PTQ4ViTPQ}                                                 & 8/8          & 79.47           & \textbf{81.48}  & \textbf{82.97}    \\
\rowcolor[rgb]{1,0.808,0.576} NoisyQuant-PTQ4ViT          & 8/8          & \textbf{79.51}\scriptsize{$\pm$0.01}  & 81.45\scriptsize{$\pm$0.02}           & 82.49\scriptsize{$\pm$0.02}             \\
\bottomrule
\end{tabular}
\endgroup
}
\vspace{-10pt}
\end{table}

\begin{table}
\centering
\caption{\textbf{PTQ accuracy of Swin models.} Abbreviations are the same as \cref{comp_vit}.}
\label{comp_swin}

\resizebox{\linewidth}{!}{
\begingroup
    \setlength{\tabcolsep}{2pt}
\begin{tabular}{cccccc} 
\toprule
\textbf{Method}                                         & \textbf{W/A} & \textbf{Swin-T} & \textbf{Swin-S} & \textbf{Swin-B} & \textbf{Swin-B*}  \\ 
\hline
Pretrained                                                      & 32/32        & 81.39           & 83.23           & 85.27           & 86.44             \\ 
\hline
Percentile~\cite{Li2019FullyQN}                                             & 6/6          & 77.75           & 80.41           & 81.90           & 82.94             \\
PTQ4VIT-Linear~\cite{Yuan2022PTQ4ViTPQ}                                                & 6/6          & 78.45           & 81.74           & 83.35           & 85.22             \\
EasyQuant~\cite{Wu2020EasyQuantPQ}                                              & 6/6          & 79.51           & 82.45           & 84.30           & 85.89             \\
\rowcolor[rgb]{1,0.808,0.576} NoisyQuant-Linear                & 6/6          & \textbf{80.01}\scriptsize{$\pm$0.06}           & \textbf{82.78}\scriptsize{$\pm$0.04}           & \textbf{84.57}\scriptsize{$\pm$0.04}           & \textbf{85.96}\scriptsize{$\pm$0.04}            \\
\hline
PTQ4VIT~\cite{Yuan2022PTQ4ViTPQ}                                                & 6/6          & 80.47           & 82.38           & 84.01           & 85.38             \\
\rowcolor[rgb]{1,0.808,0.576} NoisyQuant-PTQ4ViT & 6/6          & \textbf{80.51}\scriptsize{$\pm$0.03}  & \textbf{82.86}\scriptsize{$\pm$0.05}  & \textbf{84.68}\scriptsize{$\pm$0.06}  & \textbf{86.03}\scriptsize{$\pm$0.07}    \\ 
\hline
Percentile~\cite{Li2019FullyQN}                                             & 8/8          & 79.88           & 80.93           & 83.08           & 84.31             \\
FQ-ViT~\cite{Lin2022FQViTPQ}                                                 & 8/8          & 80.51           & 82.71           & -               & -                 \\
PTQ4VIT~\cite{Yuan2022PTQ4ViTPQ}                                                & 8/8          & 80.96           & 82.75           & 84.79           & 86.16             \\
EasyQuant~\cite{Wu2020EasyQuantPQ}                                              & 8/8          & 80.95           & 83.00           & 85.10           & 86.39             \\
\rowcolor[rgb]{1,0.808,0.576} NoisyQuant-Linear                & 8/8          & \textbf{81.05}\scriptsize{$\pm$0.03}           & \textbf{83.07}\scriptsize{$\pm$0.03}           & \textbf{85.11}\scriptsize{$\pm$0.04}          & \textbf{86.42}\scriptsize{$\pm$0.02}             \\
\hline
PTQ4VIT~\cite{Yuan2022PTQ4ViTPQ}                                                & 8/8          & 81.24           & 83.10           & 85.14           & 86.39             \\
\rowcolor[rgb]{1,0.808,0.576} NoisyQuant-PTQ4ViT & 8/8          & \textbf{81.25}\scriptsize{$\pm$0.02}  & \textbf{83.13}\scriptsize{$\pm$0.01}  & \textbf{85.20}\scriptsize{$\pm$0.03}  & \textbf{86.44}\scriptsize{$\pm$0.01}    \\
\bottomrule
\end{tabular}
\endgroup
}
\vspace{-18pt}
\end{table}

This section evaluates the impact of NoisyQuant on the PTQ performance of some commonly used vision transformer architectures. 
First, we evaluate the performance on image classification tasks with the ViT~\cite{dosovitskiy2020image}, DeiT~\cite{Touvron2021TrainingDI} and Swin~\cite{Liu2021SwinTH} transformer models. We also evaluate the performance of object detection with the DETR model~\cite{Carion2020EndtoEndOD}. Finally, we provide ablation study on how NoisyQuant improves the PTQ performance when added to each type of layers in the vision transformer model.

\subsection{Experimental settings}

\noindent\textbf{Dataset.} The classification tasks are implemented on the ImageNet-2012 classification dataset~\cite{Russakovsky2015ImageNetLS}, with 1.2 million training images and 50,000 validation images in 1000 classes. 
For object detection task, the MSCOCO 2017 dataset~\cite{Lin2014MicrosoftCC} is utilized to evaluate the PTQ performance, which contains 118,000 training images and 5,000 validation images.
For both tasks we randomly sample 1024 images from the training set as the calibration data for all PTQ methods implemented in our experiments.

\noindent\textbf{Pretrained model architecture.} We perform uniform 6-bit and 8-bit PTQ on different variants of pretrained vision transformer models. For classification tasks we quantize the model family of ViT~\cite{dosovitskiy2020image}, DeiT~\cite{Touvron2021TrainingDI}, and Swin~\cite{Liu2021SwinTH} provided by the Timm library~\cite{rw2019timm}, including large-scale models with 384$\times$384 input resolution (marked with ``*''). For detection, we perform PTQ on the official implementation of DETR~\cite{Carion2020EndtoEndOD} with ResNet-50 backbone.
We use the pretrained model checkpoint provided by the official source of each model, whose floating-point performances are reported in the results tables and match their original papers.

\noindent\textbf{Implementation Details.} Following the mainstream quantization methods~\cite{Liu2021PostTrainingQF,Yuan2022PTQ4ViTPQ}, we quantize all the weights and inputs involved in matrix multiplication. More specifically, we quantize weights of qkv-projectors, attention output projectors, MLPs, linear embeddings, and model heads. Besides, we also quantize input activations of linear layers and matrix multiplication operators. 
Following \cite{Prato2019FullyQT,Zhang2020TernaryBERTDU,Liu2021PostTrainingQF,Yuan2022PTQ4ViTPQ}, we keep layer normalization and softmax operations as full precision.
Apart from special notes, we perform symmetric layer-wise quantization for activations and symmetric channel-wise quantization for weights. In particular, weights are quantized by absolute MinMax values without clamping.
We implement NoisyQuant on top of linear quantizer EasyQuant~\cite{Wu2020EasyQuantPQ} and nonlinear quantizer PTQ4ViT~\cite{Yuan2022PTQ4ViTPQ}. After we apply the selected quantizer to determine the quantization bin width and location on calibration data, we decide the parameters of the Noisy Bias distribution through a linear search. 
We use the empirical activation quantization error as defined in~\cref{equ:obj} as the search objective, where we find the Noisy Bias parameter that can minimize the quantization error.

\subsection{NoisyQuant performance}

Here we compare NoisyQuant with state-of-the-art PTQ methods, including percentile~\cite{Li2019FullyQN}, bit-split~\cite{Wang2020TowardsAP}, Liu~\etal~\cite{Liu2021PostTrainingQF}, FQ-ViT~\cite{Lin2022FQViTPQ}, EasyQuant~\cite{Wu2020EasyQuantPQ}, and PTQ4ViT~\cite{Yuan2022PTQ4ViTPQ}, on ImageNet. 
PTQ performance of different variants of ViT, DeiT, and Swin transformer models are provided in~\cref{comp_vit},~\cref{comp_deit}, and~\cref{comp_swin}, respectively. 
Experiments are conducted for both linear and nonlinear quantizers.

\vspace{-2pt}
\noindent\textbf{Linear quantizer.} Previous linear PTQ methods suffer from severe performance degradation on vision transformer activation quantization. 
EasyQuant~\cite{Wu2020EasyQuantPQ} achieves the best performance among linear quantizers, yet still induce 2-6\% of accuracy drop under 6-bit quantization compared to the floating-point baseline. 
By applying the proposed NoisyQuant method on top of the EasyQuant quantizer, the resulted NoisyQuant-Linear quantizer consistently and significantly improves the PTQ accuracy on all variants of vision transformer models. Under the W6A6 setting, NoisyQuant-Linear achieves performance improvement of 1.73\% on ViT-S, 0.98\% on ViT-B*, 1.1\% on DeiT-S, and 0.5\% on SWIN-T over EasyQuant. 
Under the W8A8 setting, as the performance gap between EasyQuant and floating-point gets smaller, NoisyQuant still appears beneficial in further improving the PTQ performance.
It is worth noting that although nonlinear quantizer like PTQ4ViT~\cite{Yuan2022PTQ4ViTPQ} consistently achieves higher PTQ performance than the linear EasyQuant quantizer, the enhancement brought by NoisyQuant enables the NoisyQuant-Linear quantizer to achieve on-par or even better performance than the nonlinear PTQ4ViT, with much lower overhead in hardware deployment.

\noindent\textbf{Nonlinear quantizer.} As we claim NoisyQuant as a quantizer-agnostic enhancement of PTQ, we further implement NoisyQuant on top of the nonlinear quantizer PTQ4ViT~\cite{Yuan2022PTQ4ViTPQ}, namely NoisyQuant-PTQ4ViT.
NoisyQuant-PTQ4ViT outperforms most results of PTQ4ViT, notably achieving for a 1.25\% improvement on DeiT-S, 0.67\% on ViT-S, and 0.67\% on Swin-B in the W6A6 setting. The W6A6 PTQ performance of Swin-S and Swin-B* for the first time improved to within 0.5\% of the floating-point baseline by the NoisyQuant-PTQ4ViT quantizer, which can never be achieved without the help of NoisyQuant. 

Besides classification models, we report the PTQ performance of the DETR object detection model in~\cref{DETR}. NoisyQuant implemented on top of the EasyQuant quantizer also outperforms the EasyQuant baseline, and all previous linear PTQ quantizers including percentile\cite{Li2019FullyQN}, bit-split\cite{Wang2020TowardsAP}, and Liu~\etal~\cite{Liu2021PostTrainingQF}.

\begin{table}
\centering
\caption{\textbf{PTQ accuracy of DETR models.} Mean average precision (mAP) evaluated on the MSCOCO 2017 dataset is reported. }
\label{DETR}
\resizebox{0.8\linewidth}{!}{
\begingroup
    \setlength{\tabcolsep}{3pt}
\begin{tabular}{l|ccc} 
\toprule
\multicolumn{1}{l}{\textbf{Model}}                                        & \textbf{Method}          & \textbf{W/A}            & \textbf{mAP}              \\ 
\hline
\multirow{6}{*}{\begin{tabular}[c]{@{}l@{}}DETR~\cite{Carion2020EndtoEndOD}\\(COCO2017)\end{tabular}} & Pretrained                 & 32/32                   & 42.0                      \\ 
\cline{2-4}
                                                                          & Percentile~\cite{Li2019FullyQN}                & 8/8                     & 38.6                      \\
                                                                          & Bit-Split~\cite{Wang2020TowardsAP}                & 8/8                     & 40.6                      \\
                                                                          & \multicolumn{1}{c}{Liu~\etal~\cite{Liu2021PostTrainingQF}} & \multicolumn{1}{c}{8/8} & \multicolumn{1}{c}{41.2}  \\
                                                                          & EasyQuant~\cite{Wu2020EasyQuantPQ}               & 8/8                     & 41.1                      \\
                                                                          & {\cellcolor[rgb]{1,0.808,0.576}}NoisyQuant-Linear        & {\cellcolor[rgb]{1,0.808,0.576}}8/8                     & {\cellcolor[rgb]{1,0.808,0.576}}\textbf{41.4}\scriptsize{$\pm$0.05}             \\
\bottomrule
\end{tabular}
\endgroup
}
\end{table}

\subsection{NoisyQuant's impact on different layer types}

As we show the effectiveness of NoisyQuant in improving the PTQ performance of the entire model, here we explore if NoisyQuant is helpful for all types of layers in the vision transformer model. \cref{noisy_layerwise} exhibits the effect of applying NoisyQuant on different types of layers. The checkmark in the table means that we apply NoisyQuant to the input activation of all the layers of that particular type, otherwise the layer is quantized with EasyQuant only. For all layer types, applying NoisyQuant consistently improves the PTQ performance of the entire model. Specifically, 'fc2' layers derive the greatest benefit from NoisyQuant, which corresponds to our previous analysis that the activation distribution after GELU function brings the major challenge to PTQ, which can be resolved with the proposed NoisyQuant method. After adding NoisyQuant to the input activation of all linear layers, the model achieves the maximum performance boost.

\begin{table}
\centering
\caption{\textbf{Applying NoisyQuant on different layer types.} Starting from W6A6 PTQ with the EasyQuant quantizer, we add NoisyQuant to each type of layers in the transformer model.}
\label{noisy_layerwise}
\resizebox{0.8\linewidth}{!}{
\begingroup
    \setlength{\tabcolsep}{3pt}
\begin{tabular}{cccccc} 
\toprule
\textbf{Model}          & \begin{tabular}[c]{@{}c@{}}\textbf{qkv}\\\textbf{noise}\end{tabular} & \begin{tabular}[c]{@{}c@{}}\textbf{proj}\\\textbf{noise}\end{tabular} & \begin{tabular}[c]{@{}c@{}}\textbf{fc1}\\\textbf{noise}\end{tabular} & \begin{tabular}[c]{@{}c@{}}\textbf{fc2}\\\textbf{noise}\end{tabular} & \begin{tabular}[c]{@{}c@{}}\textbf{Top-1}\\\textbf{W6A6}\end{tabular}  \\ 
\hline
\multirow{6}{*}{DeiT-S~\cite{Touvron2021TrainingDI}} & \textcolor{black}{\ding{55}}                                         & \textcolor{black}{\ding{55}}                               & \textcolor{black}{\ding{55}}                                       & \textcolor{black}{\ding{55}}                                       & 75.27                                                \\
                        & \textcolor{green!66!blue}{\ding{51}}                                                        & \textcolor{black}{\ding{55}}                                                                  & \textcolor{black}{\ding{55}}                                                                   & \textcolor{black}{\ding{55}}                                                                    & 75.38                                                                  \\
                        & \textcolor{black}{\ding{55}}                                                                     & \textcolor{green!66!blue}{\ding{51}}                                                          & \textcolor{black}{\ding{55}}                                                                   & \textcolor{black}{\ding{55}}                                                                   & 75.45                                                                  \\
                        & \textcolor{black}{\ding{55}}                                                                     &  \textcolor{black}{\ding{55}}                                                                     & \textcolor{green!66!blue}{\ding{51}}                                                         & \textcolor{black}{\ding{55}}                                                                  & 75.33                                                                  \\
                        & \textcolor{black}{\ding{55}}                                                                     & \textcolor{black}{\ding{55}}                                                                       & \textcolor{black}{\ding{55}}                                                                     & \textcolor{green!66!blue}{\ding{51}}                                                         & 76.21                                                                  \\
                        & \textcolor{green!66!blue}{\ding{51}}                                                         & \textcolor{green!66!blue}{\ding{51}}                                                          & \textcolor{green!66!blue}{\ding{51}}                                                         & \textcolor{green!66!blue}{\ding{51}}                                                         & \textbf{76.37}                                                                  \\ 
\hline
\multirow{6}{*}{Swin-T~\cite{Liu2021SwinTH}} & \textcolor{black}{\ding{55}}                                                                     & \textcolor{black}{\ding{55}}                                                                      & \textcolor{black}{\ding{55}}                                                                     &  \textcolor{black}{\ding{55}}                                                                    & 79.51                                                                  \\
                        & \textcolor{green!66!blue}{\ding{51}}                                                         &  \textcolor{black}{\ding{55}}                                                                     &  \textcolor{black}{\ding{55}}                                                                    & \textcolor{black}{\ding{55}}                                                                     & 79.53                                                                  \\
                        & \textcolor{black}{\ding{55}}                                                                     & \textcolor{green!66!blue}{\ding{51}}                                                          &  \textcolor{black}{\ding{55}}                                                                    & \textcolor{black}{\ding{55}}                                                                     & 79.56                                                                  \\
                        & \textcolor{black}{\ding{55}}                                                                     &  \textcolor{black}{\ding{55}}                                                                    & \textcolor{green!66!blue}{\ding{51}}                                                         & \textcolor{black}{\ding{55}}                                                                     & 79.52                                                                  \\
                        & \textcolor{black}{\ding{55}}                                                                  &  \textcolor{black}{\ding{55}}                                                                     & \textcolor{black}{\ding{55}}                                                                     & \textcolor{green!66!blue}{\ding{51}}                                                         & 79.80                                                                  \\
                        & \textcolor{green!66!blue}{\ding{51}}                                                         & \textcolor{green!66!blue}{\ding{51}}                                                         & \textcolor{green!66!blue}{\ding{51}}                                                 & \textcolor{green!66!blue}{\ding{51}}                                                    & \textbf{80.01}                                                                  \\
\bottomrule
\end{tabular}
\endgroup
}
\vspace{-12pt}
\end{table}

\section{Conclusions}
\label{sec:con}

This work proposes NoisyQuant, a noisy bias-enhanced post-training activation quantization method for the complicated activation distribution of vision transformer models. We theoretically prove the quantization error reduction brought by adding noisy bias to the input activation before quantization, and empirically show the effectiveness of NoisyQuant on both linear and nonlinear quantizers on vision transformer models.
We hope this work opens up a new direction to reduce PTQ quantization error through actively changing the distribution being quantized, and inspires better Noisy Bias distribution design or generation method based on quantizer and activation characteristics.\\
\small{\textbf{Acknowledgement}
This work was supported in part by the National Key Research and Development Program of China under Grant 2022YFB4400900.
The Berkeley team acknowledges support from Berkeley Deep Drive, Intel Corporation, and Panasonic.}

{\small
\bibliographystyle{ieee_fullname}
\bibliography{egbib}
}

\clearpage

\appendix
This document provides additional visualizations and experimental results to support the main paper. We demonstrate results of quantization error on model output in \cref{ap:output_qe_table}, visualize prediction scores for ImageNet classes in \cref{ap:output_vis}, illustrate more histogram examples of the input and output activation distributions on different transformer layers in \cref{ap:vis_input_output}, discuss memory and computation overhead in \cref{ap:overhead}, and show additional experimental results in \cref{ap:add_exp}.

\section{Quantization error of model output}
\label{ap:output_qe_table}

In this section, we show the comparison of the output logits between EasyQuant~\cite{Wu2020EasyQuantPQ} and NoisyQuant with 6-bit ViT~\cite{dosovitskiy2020image}, DeiT~\cite{Touvron2021TrainingDI} and Swin~\cite{Liu2021SwinTH} models. Here NoisyQuant is implemented on top of EasyQuant with the proposed noisy bias enhancement. We go through the whole ImageNet validation set and calculate the mean-square error of model output logits on each quantized model compared to the pretrained floating-point counterpart. As shown in~\cref{tab:model_out_qe}, NoisyQuant achieves quantization error reduction on all model outputs, especially for ViT-S (17\%) and Swin-T (16\%) models.

\begin{table}[h]
\centering
\caption{\textbf{Quantization error of model output.} Models are quantized by EasyQuant and NoisyQuant with the W6A6 setting.}
\label{tab:model_out_qe}
\resizebox{\linewidth}{!}{
\begin{tabular}{lccl} 
\toprule
\textbf{Model}   & \textbf{EasyQuant}~\cite{Wu2020EasyQuantPQ} & \textbf{NoisyQuant} & \textbf{Reduction}     \\ 
\hline
\textbf{ViT-S}   & 1.0400             & 0.8583              & 0.1818 \textbf{(17\%)}  \\
\textbf{ViT-B}   & 0.6365             & 0.5982              & 0.0383 \textbf{(~6\%)}   \\
\textbf{ViT-B*}  & 0.6956             & 0.6360              & 0.0596 \textbf{(~9\%)}   \\ 
\hline
\textbf{DeiT-S}  & 0.3270             & 0.2934              & 0.0335\textbf{ (10\%)}  \\
\textbf{DeitT-B} & 0.2869             & 0.2584              & 0.0284 \textbf{(10\%)}  \\
\textbf{DeiT-B*} & 0.1984             & 0.1760              & 0.0224\textbf{ (11\%)}  \\ 
\hline
\textbf{Swin-T}  & 0.0913             & 0.0765              & 0.0148 \textbf{(16\%)}  \\
\textbf{Swin-S}  & 0.0296             & 0.0289              & 0.0007 \textbf{(~2\%)}   \\
\textbf{Swin-B}  & 0.0505             & 0.0457              & 0.0047 \textbf{(~9\%)}   \\
\textbf{Swin-B*} & 0.0412             & 0.0399              & 0.0013 \textbf{(~3\%)}   \\
\bottomrule
\end{tabular}
}
\vspace{-20pt}
\end{table}

\section{Visualization of model output}
\label{ap:output_vis}

Following \cref{ap:output_qe_table}, we visualize model output in~\cref{fig:model_output} to give further perspectives on how the reduced quantization error achieved by NoisyQuant improved final accuracy. Specifically, we plot prediction logits produced by the floating-point (red), EasyQuant (gray), and NoisyQuant (green) models on the 1000 ImageNet~\cite{Russakovsky2015ImageNetLS} classes, respectively. The highest logits are marked with stars, where the location of the red star corresponds to the ground truth class.
With less quantization error, NoisyQuant logits closely match that of the floating-point model, thus achieving better performance than EasyQuant.

\begin{table}[ht]
\centering
\caption{Comparing to reparameterization.}
\label{tab:add_baselines}
\begin{tabular}{c|c|c|c} 
\toprule
Model  & W/A & Reparam. & \textbf{NoisyQuant}             \\ 
\hline
ViT-S  & 6/6 & 76.66   & \textbf{78.90} $\pm$ 0.06  \\
DeiT-B & 6/6 & 81.03   & \textbf{81.26} $\pm$ 0.04  \\
Swin-S & 6/6 & 82.46   & \textbf{82.83} $\pm$ 0.04  \\
\bottomrule
\end{tabular}
\end{table}

\begin{figure*}[ht]
    \centering
    \begin{minipage}[t]{\linewidth}
        \includegraphics[width=0.98\linewidth]{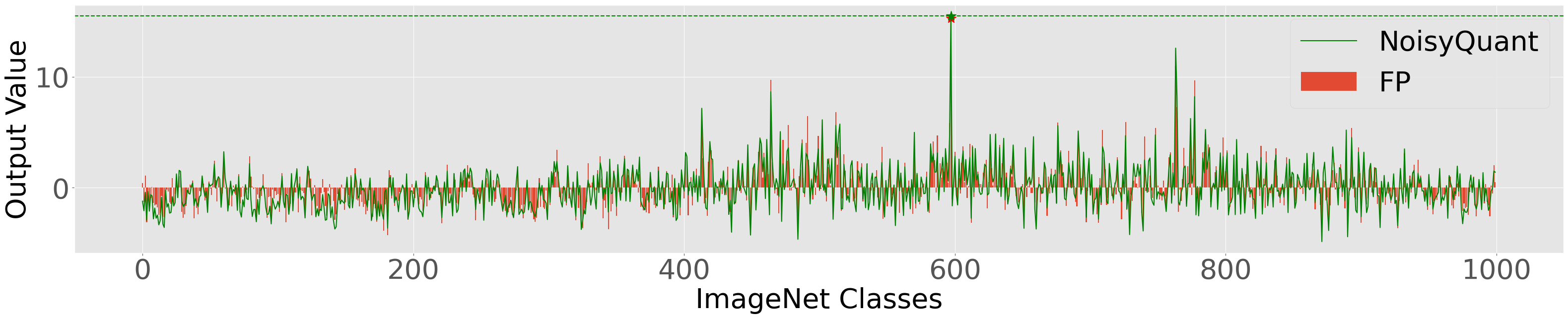}
        \includegraphics[width=0.98\linewidth]{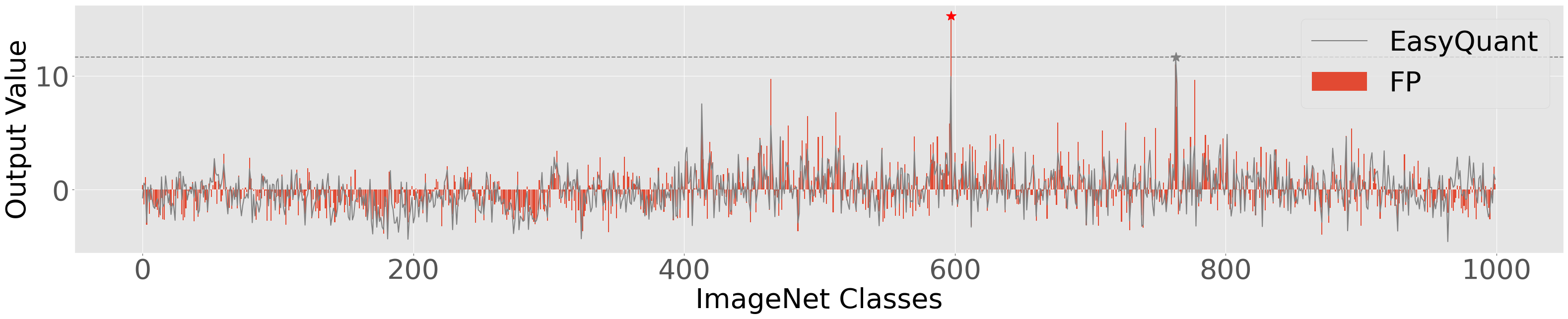}
            \subcaption{Ground truth = 597} 
    \end{minipage}
    \begin{minipage}[t]{\linewidth}
        \includegraphics[width=0.98\linewidth]{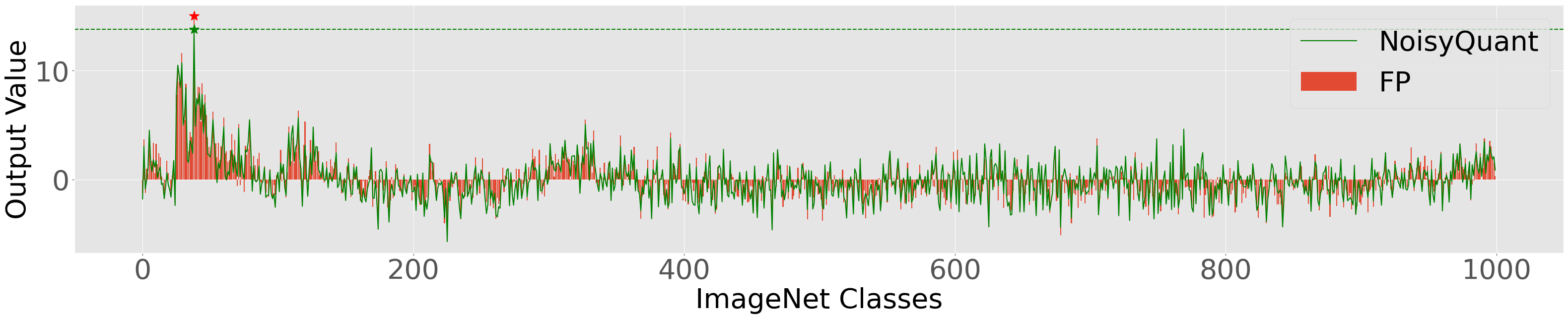}
        \includegraphics[width=0.98\linewidth]{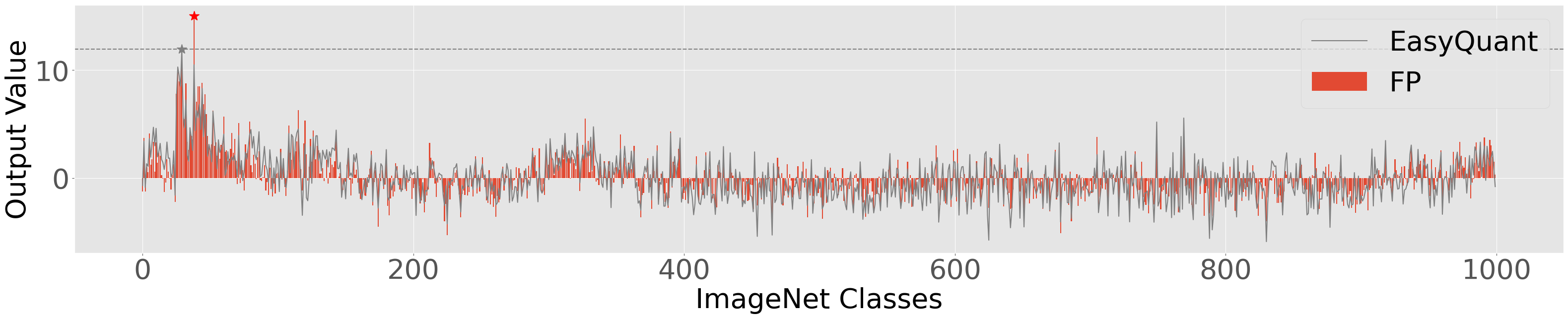}
            \subcaption{Ground truth = 38} 
    \end{minipage}
    \caption{Model output of the floating-point (red), EasyQuant (green), and NoisyQuant (gray) model. The floating-point and NoisyQuant models give correct predictions (red/green star) while EasyQuant gives wrong prediction (gray star).}
    \label{fig:model_output}
\end{figure*}

\section{Additional input and output activation histogram}
\label{ap:vis_input_output}
In this section, we present more histogram examples of layer input and output as previously described in Sec. 4.2 of the main paper. 
We briefly illustrate the pipeline of EasyQuant and NoisyQuant in \cref{fig:addVis}. The top-left sub-figure refers to input activation $X$, and EasyQuant follows the gray arrow while NoisyQuant follows the blue. NoisyQuant utilizes the proper-selected noisy bias $N$ to refine the input before quantization (shown in the bottom-left sub-figure). The output histograms are shown in the right sub-figures, and we point out the mismatch caused by EasyQuant with the orange arrow.

As we have emphasized in the main paper, transformer layers produce sophisticated activation distributions. \cref{fig:addVis} gives more examples from different transformer layers. \cref{fig:addVis-1} and \cref{fig:addVis-2} show  fc2 layers in ViT-S and DeiT-S which takes GELU~\cite{hendrycks2016gaussian} activations as input; the asymmetric and heavy-tailed input activation distribution makes a negative impact on the layer output produced by EasyQuant. Instead, NoisyQuant refines the distribution to achieve a better match in the quantized layer output.
\cref{fig:addVis-3} gives an example of the downsample layer in Swin models which as well enjoys the noisy bias enhancement.

\begin{figure*}[ht]
    \centering
    \begin{minipage}[t]{\linewidth}
        \includegraphics[width=0.9\linewidth]{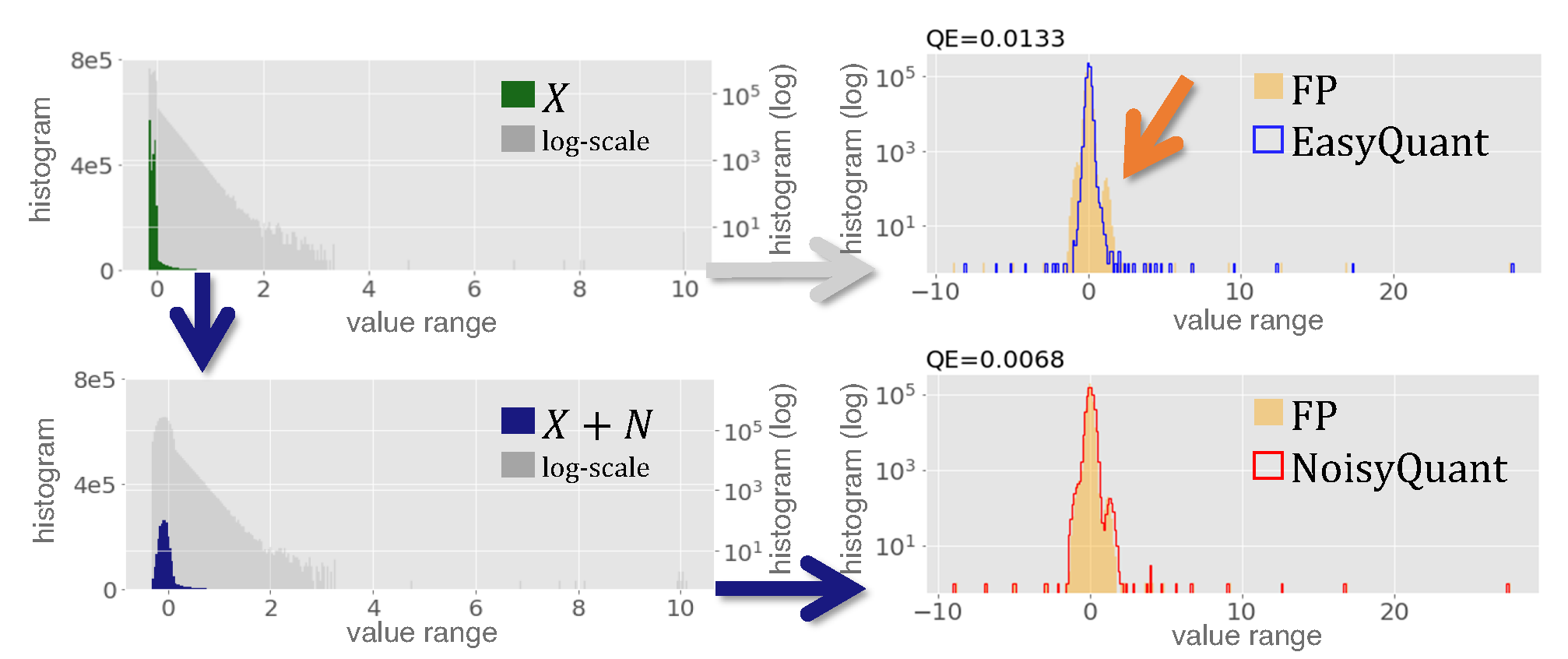}
            \subcaption{fc2 layer in ViT-S}
            \label{fig:addVis-1}
    \end{minipage}
    \begin{minipage}[t]{\linewidth}
        \includegraphics[width=0.9\linewidth]{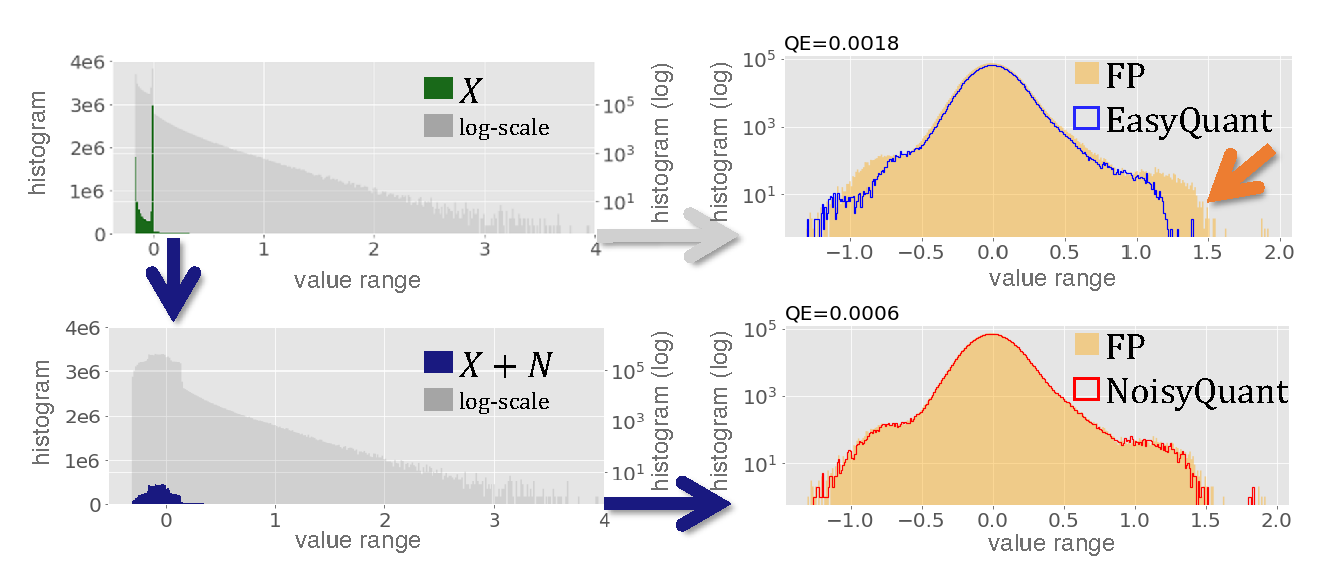}
            \subcaption{fc2 layer in DeiT-S}
            \label{fig:addVis-2}
    \end{minipage}
    \begin{minipage}[t]{\linewidth}
        \includegraphics[width=0.9\linewidth]{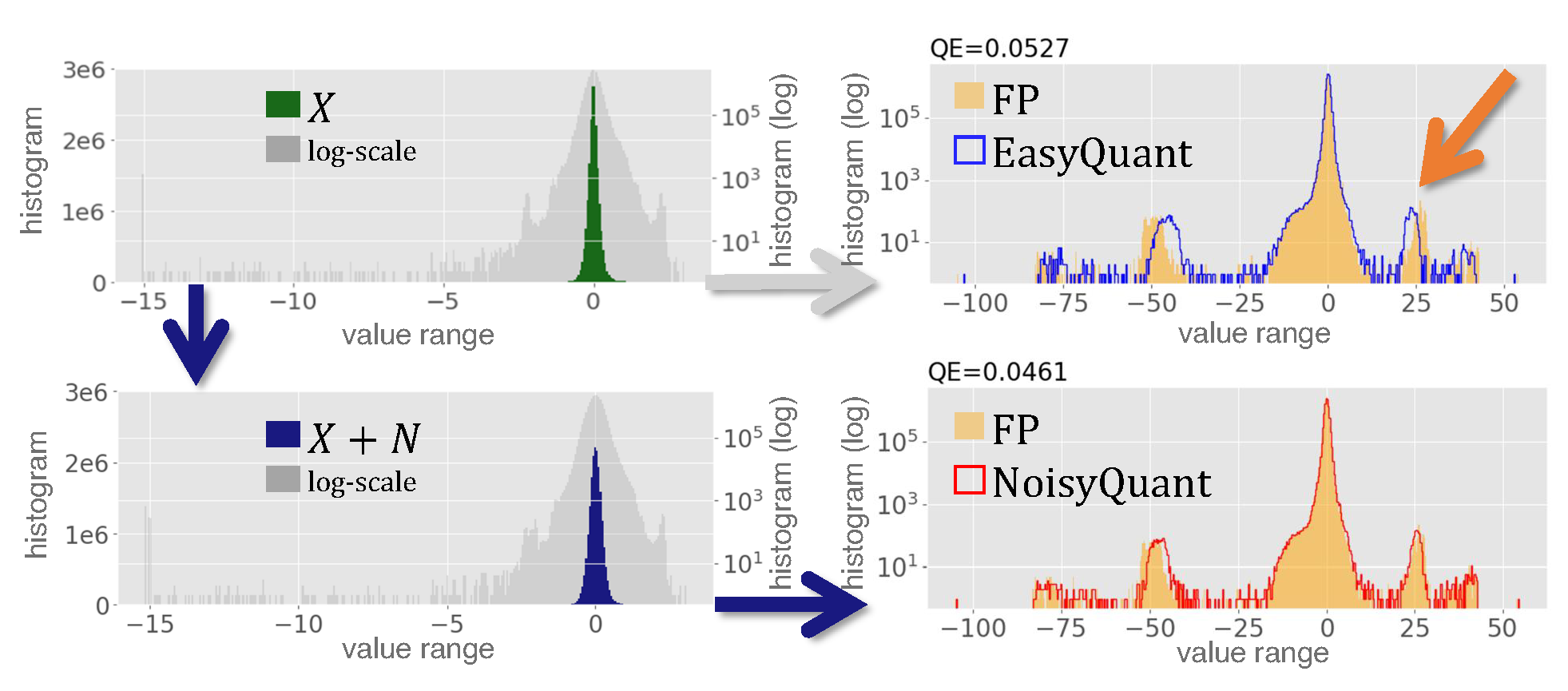}
            \subcaption{downsample layer in Swin-B}
            \label{fig:addVis-3}
    \end{minipage}
    \caption{Input (left) and output (right) histogram on different layers.}
    \label{fig:addVis}
\end{figure*}

\begin{table*}[h]
\centering
\caption{Performance with smaller calibration set.}
\label{tab:calibsize}
\resizebox{1.0\linewidth}{!}{
\begin{tabular}{cc|cccccccccc}
\toprule
\textbf{Size} & \textbf{W/A}  & \textbf{ViT-S} & \textbf{ViT-B} & \textbf{ViT-B*} & \textbf{DeiT-S} & \textbf{DeiT-B} & \textbf{DeiT-B*} & \textbf{Swin-T} & \textbf{Swin-S} & \textbf{Swin-B}  & \textbf{Swin-B*}\\
\midrule
32   & 6/6 & 76.81 & 81.89 & 82.81 & 76.30 & 79.71 & 81.19 & 79.97 & 82.74 & 84.55 & 85.90  \\
128  & 6/6 & 76.87 & 81.97 & 82.86 & 76.47 & 80.20 & 81.24 & 80.13 & 82.68 & 84.44 & 86.00 \\
1024 & 6/6 & 76.86 &	81.90 &	83.00 &	76.37 &	79.77 &	81.40 &	80.01 &	82.78 &	84.57 &	85.90 \\
\bottomrule
\end{tabular}
}
\end{table*}

\section{Memory and computation overhead}
\label{ap:overhead}
\noindent\textbf{Memory overhead.} In practice, for weights $W\in\mathbb{R}^{k\times m}$ and activations $X\in\mathbb{R}^{m\times n}$, we follow the standard implementation to set bias $B\in \mathbb{R}^{k\times 1}$ and sample noise $N\in \mathbb{R}^{m\times 1}$, so the denoising bias $B' = B-q_W(W)N$ is also $\mathbb{R}^{k\times 1}$, where $q_W(\cdot)$ is the quantizer. The sum follows the broadcasting rule. Storing $N$ brings minimal overhead, for instance, DeiT-B* has 86.9M params, with only 0.06M (0.07\%) for storing the noise.\\
\noindent\textbf{Computation overhead.} The matrix multiplication, i.e., $WX$, dominates the computation of ViT linear layers, requiring $10^3\times$ more MAC than the number of adds in $X+N$ and bias. So the cost of FP32 add is negligible ($<$0.4\%) to that of INT8 layer.
Further, $N$ and $B'$ can be INT16 rather than FP32, enabling integer-only inference and reducing the cost of add to $<$0.03\%. We observe no accuracy differences in using INT16 or FP32 for $N$ and $B'$ in our experiments. We estimate the energy cost with 0.23pJ/Int8-MAC, 0.9pJ/FP32-Add, and 0.05pJ/Int16-Add following \cite{horowitz20141}.

\section{Additional experiments}
\label{ap:add_exp}
\noindent\textbf{Ablation study on calibration size.} We follow \cite{Liu2021PostTrainingQF}'s setting for calibration size 1024. Further experiments show that calibration size as low as 32 can still produce similar performance (see \cref{tab:calibsize}).\\
\noindent\textbf{Additional baselines.} Concurrent works \cite{li2022repq,wei2022outlier} introduce the reparameterization approach which reparameterizes LN layer to suppress outliers by scaling down activation values. NoisyQuant is orthogonal as we actively change the activation distribution being quantized without scaling. So NoisyQuant can be plugged in after reparameterization. We reproduce the reparameterization used in the two works and subsequently add NoisyQuant to show consistent improvement in~\cref{tab:add_baselines}.\\
\noindent\textbf{Additional models.} Beyond ViT, on ResMLP-24 with W6A6, NoisyQuant (76.71\%) beat EasyQuant (76.48\%) by 0.23\%.\\

\end{document}